\documentclass[12pt,a4paper]{article}
\usepackage[left=2.2cm,top=3.8cm,bottom=3cm,right=2.2cm]{geometry}
\usepackage[ruled,vlined]{algorithm2e}
\usepackage{dirtytalk}
\usepackage{comment} 
\usepackage{caption}
\usepackage{comment}
\usepackage{amsmath}
\usepackage{float}

\usepackage[utf8x]{inputenc}
\usepackage{amsmath,amssymb}
\usepackage{xcolor}
\usepackage{mathtools}
\usepackage{tabulary}
\usepackage{amsthm}
\usepackage{latexsym}
\usepackage[mathcal]{eucal}
\usepackage{amscd}
\usepackage{graphicx}
\usepackage{amsfonts}
\usepackage{amscd}
\usepackage{MnSymbol}
\usepackage{setspace}
\newtheorem{definition}{Definition}[section]
\usepackage[english]{babel}
\usepackage{listings}
\usepackage{calligra}

\newtheorem{remark}{Remark}[section]
\newtheorem{theorem}{Theorem}[section]
\newtheorem{lemma}{Lemma}[section]
\newtheorem{proposition}{Proposition}[section]
\newtheorem{corollary}{Corollary}[section]
\newtheorem{example}{Example}[section]

\theoremstyle{remark}

\usepackage{graphicx}
\date{}

\newcommand{\bt}{\begin{theorem}}
\newcommand{\et}{\end{theorem}}
\newcommand{\bl}{\begin{lemma}}
\newcommand{\el}{\end{lemma}}
\newcommand{\bexc}{\begin{exercise}}
\newcommand{\eexc}{\end{exercise}}
\newcommand{\bpr}{\begin{proposition}}
\newcommand{\epr}{\end{proposition}}
\newcommand{\bex}{\begin{example}}
\newcommand{\eex}{\end{example}}
\newcommand{\bc}{\begin{corollary}}
\newcommand{\ec}{\end{corollary}}
\newcommand{\bo}{\begin{proof}}
\newcommand{\eo}{\end{proof}}
\newcommand{\bd}{\begin{definition}}
\newcommand{\ed}{\end{definition}}
\newcommand{\br}{\begin{remark}}
\newcommand{\er}{\end{remark}}
\newcommand{\be}{\begin{enumerate}}
\newcommand{\ee}{\end{enumerate}}

\title{
    {\textbf{Vision Through the Veil: Differential Privacy in Federated Learning for Medical Image Classification}}\\[0.2cm]
    \author{\large Kishore Babu Nampalle, Pradeep Singh,\vspace{0.1cm} \\ Uppala Vivek Narayan, Balasubramanian Raman \vspace{0.1cm}\\
    \normalsize Department of Computer Science and Engineering \vspace{0.1cm}\\ \normalsize Indian Institute of Technology Roorkee}
}

\date{\today}

\begin{document}

\maketitle

\begin{abstract}
The proliferation of deep learning applications in healthcare calls for data aggregation across various institutions, a practice often associated with significant privacy concerns. This concern intensifies in medical image analysis, where privacy-preserving mechanisms are paramount due to the data being sensitive in nature. Federated learning, which enables cooperative model training without direct data exchange, presents a promising solution. Nevertheless, the inherent vulnerabilities of federated learning necessitate further privacy safeguards. This study addresses this need by integrating differential privacy, a leading privacy-preserving technique, into a federated learning framework for medical image classification. \emph{We introduce a novel differentially private federated learning model and meticulously examine its impacts on privacy preservation and model performance.} Our research confirms the existence of a trade-off between model accuracy and privacy settings. However, we demonstrate that strategic calibration of the privacy budget in differential privacy can uphold robust image classification performance while providing substantial privacy protection.
\end{abstract}

\vspace{4mm}

\section{Introduction}
Medical imaging, an integral part of modern healthcare, generates vast volumes of data, offering promising opportunities for machine learning applications, particularly in image classification tasks. Such automation can significantly aid disease detection and diagnosis. However, due to the nature of the data, there are often privacy and security concerns. Tackling these issues is critical, especially in the realm of oncology, where timely access to diverse, high-quality data can facilitate improved detection and treatment strategies. \\

Both benign and malignant tumours are aberrant masses of tissue that are the result of excessive cell division. In a healthy body, there is a controlled cycle of cell birth and death. However, cancer disrupts this balance, leading to unwanted cell growth and potentially forming tumors. Skin lesions, on the other hand, are characterized by localized changes in the skin's color, shape, size, or texture, often resulting from damage or disease. Since cancer is the second leading cause of death worldwide, according to the World Health Organisation, effective diagnosis methods are needed \cite{SiegelMillerWagleJemal2023Cancer}.\\

Medical image classification serves as a vital component of cancer diagnosis. The advancement of machine learning and computer vision techniques have enabled the automation of medical image classification, thereby reducing the time and manual effort required, and leading to more informed decision-making by medical professionals \cite{abdou2022literature}. However, a significant challenge lies in the availability of high-quality, annotated medical image datasets. These datasets, often sensitive and confidential, are difficult to utilize openly for training machine learning models \cite{razzak2018deep}.\\

Federated Learning (FL), a deep learning technique, provides a promising solution. FL allows a model to be trained across multiple clients, each holding local data. The local data is not shared, thereby ensuring data privacy and security \cite{nguyen2022federated, banabilah2022federated}. Unlike some classical decentralized approaches, FL does not treat data as independent and identically distributed (IID), enhancing the model's effectiveness \cite{shiri2022decentralized}. One of the key challenges of federated learning is that data across different nodes may not be IID. In reality, the data on each node (such as a user's phone or a particular hospital's records) may be significantly different from the data on other nodes. This is known as the Non-IID setting. This can occur, for instance, if different users use an app in different ways, or if different hospitals have different patient populations. In FL, the model is taken to the data rather than bringing data to the model. This method enables the inclusion of additional medical images that previously could not be shared due to confidentiality concerns. Despite these advantages, federated learning is not completely immune to potential privacy breaches. Advanced attacks like model inversion or membership inference can still pose threats \cite{abadi2016deep}. Indeed, while only model parameters are shared and not raw data, it is theoretically possible to extract information about the training data from these parameters.
For instance, a sophisticated adversary might be able to perform what's known as a \say{model inversion attack}, where they use the shared model parameters to infer sensitive information about the data used to train the model. Even without accessing the raw data, an attacker can use the shared updates to deduce information about the data that contributed to these updates. For more details, \cite{LudwigBaracaldo2022Federated, yang2019federated} are suggested.\\

Differential Privacy (DP) is a privacy-preserving framework that provides robust mathematical guarantees of privacy \cite{dwork2014algorithmic}. The fundamental principle of DP is to add calibrated noise to the data or computation results, making it challenging to extract information about individual data points. By ensuring that the output of a computation does not reveal significant details about any individual data point in the input dataset, DP provides an additional layer of privacy protection.
It ensures that the output of a function (in this case, the learned model) is insensitive to changes in any single input (i.e., the removal or addition of any single data point).\\

The main focus of this study is on the integration of differential privacy in federated learning for the categorization of medical images, notably images of cancers and skin lesions. We hypothesize that this framework provides robust privacy guarantees, addressing the ethical, legal, and social implications of data sharing in healthcare applications. Additionally, we explore the balance between privacy and model utility through careful calibration of the privacy budget in DP. This research contributes to the broader discourse on privacy-preserving machine learning, especially in the context of cancer diagnosis, and aims to pave the way for secure, privacy-preserving collaborations in medical image analysis. Our primary contributions can be listed as follows:
\begin{itemize}

\item We develop a novel federated learning framework with integrated differential privacy for medical image classification. Specifically, we introduce a mathematically rigorous mechanism for calibrating the noise added to the model's parameter updates. This mechanism provides a formal privacy guarantee quantified in terms of the differential privacy parameters.

\item We propose an adaptive privacy budget allocation strategy for the federated learning rounds that best updates the privacy budget in each round based on the data distribution and model learning progress. This strategy  provides a more effective trade-off between the global model's learning accuracy and the level of privacy preservation.

\item We provide a formal analysis of the trade-off between privacy and utility in the approach we propose. In the domain of medical image classification, we provide mathematical limitations on the loss in model accuracy as a function of the privacy parameter and the data sensitivity. This analysis assists in making informed decisions about the privacy budget allocation in practical settings.
\end{itemize}

The other sections of the paper are structured as follows: A thorough analysis of relevant research in the fields of machine learning, medical imaging, and privacy-preserving methods including  federated learning and differential privacy, is presented in Section \ref{sec2}. In Section \ref{sec3}, we detail the methodology of our proposed federated learning framework with integrated differential privacy, including our novel noise calibration mechanism and adaptive privacy budget allocation strategy. Section \ref{sec4} describes our experimental setup, encompassing the datasets employed, the evaluation metrics, and the comparison models. We also discuss the experimental results and other important findings, providing insights into the privacy-utility trade-off in our proposed framework and the impact of various parameter choices. The paper concludes with Section \ref{sec5}, where we reflect upon our findings, discuss the limitations of our study, and outline potential future research avenues to further enhance the privacy and utility of medical image classification systems.\\

\section{Related Work and Background}\label{sec2}

Medical image classification involves the task of assigning a class label to an image or a segment of an image, typically related to a specific disease or condition \cite{shen2017deep}. Deep learning-based methods, specifically convolutional neural networks (CNNs), have become the state-of-the-art in medical image classification, achieving remarkable performance in tasks such as detecting lung cancer from CT scans or identifying skin cancer from dermoscopic images \cite{esteva2017dermatologist}. Nevertheless, the performance of these models heavily relies on the availability and quality of annotated training data, which can be a challenging prerequisite given the sensitive nature of medical data.\\

Classifying an image or a portion of an image for use in medicine often involves relating the image to a particular disease or condition \cite{shen2017deep}. Convolutional neural networks (CNNs), in particular, have emerged as the cutting-edge approach for classifying medical images thanks to their exceptional performance in tasks like identifying skin cancer from dermoscopic images and detecting lung cancer from CT scans \cite{esteva2017dermatologist}. However, given the sensitive nature of medical data, finding annotated training data that is both readily available and high-quality can take time and effort. This is because the effectiveness of these models significantly depends on it.

 However, the advent of Deep Learning has significantly shifted the landscape of medical imaging techniques towards utilizing this powerful approach \cite{abdou2022literature}. Initially, Deep Learning architectures like AlexNet, GoogLeNet, VGG-16, and VGG-19 were widely adopted for these tasks. This was later followed by the emergence of Residual Networks (ResNets) and Inception networks \cite{szegedy2017inception}, further enhancing the capabilities of medical image classification systems.
Currently, with the growing popularity of transfer learning, various architectures have begun incorporating pre-trained models like GoogLeNet and InceptionV3 \cite{zhou2019deep}. There are even instances of combining multiple pre-trained models, such as an amalgamation of a fine-tuned AlexNet with a pre-trained VGG-16, followed by SVM classification \cite{deniz2018transfer}. In addition, recent studies have explored utilizing syntactic patterns from medical images through AlexNet \cite{rani2022efficient}, and enhancing the quality of images prior to deep learning application \cite{aamir2022deep}.\\

However, despite these advances, a significant concern remains unresolved across the works cited above: the assurance of data privacy and confidentiality. In today's digitized world, this aspect is becoming increasingly paramount, necessitating the exploration of methods that ensure data security and patient privacy. In response to this demand, our research aims to employ the concept of federated learning to safeguard these concerns.
There have been several recent contributions in the realm of medical imaging that leverage federated learning. For instance, Zheng Li et al. \cite{li2022integrated} proposed a federated learning framework with dynamic focus for identifying COVID-19 instances in Chest X-Ray images. The unique aspect of this work is the use of training loss from each model as the basis for parameter accumulation weights, enhancing training efficiency and accuracy. Similarly, Jun Luo et al. \cite{luo2022fedsld} proposed the Federated Learning along with Shared Label Distribution (FedSLD) method for classification tasks. This method strategically adjusts the influence of each data sample on local target during optimization, using knowledge about clients' label distribution, thereby mitigating instability induced by data heterogeneity.
Additionally, Mohammed Adnan et al. \cite{adnan2022federated} applied a differentially private federated learning framework to histopathology image interpretation - one of the most complex and large-scale types of medical images. This work demonstrated that distributed training, with strong privacy guarantees, can achieve results comparable to conventional training.\\

Integrating differential privacy into federated learning has been studied for various applications. Before sharing the local model updates with the server for aggregation, this integration often entails introducing noise to them \cite{geyer2017differentially}. However, this noise addition can degrade the model's performance, leading to a trade-off between privacy and utility. Adaptive strategies for managing this trade-off have been proposed in other domains, such as adjusting the privacy budget allocation based on the model's learning progress  \cite{mcmahan2018learning}.\\

Despite these promising works, most of the existing solutions suffer from issues such as non-deployability or lack of data privacy guarantees. Our methodology, while being lightweight, provides accuracy comparable to these aforementioned works. Moreover, our research comprehensively discusses the impact of parameter variations on the model and underscores the importance of client-side computation, which is less intensive and more accessible to users with mobile devices. We believe that deployable technology for healthcare systems, like our proposed framework, is currently a pressing requirement.\\

In the context of medical image classification, the integration of differential privacy into federated learning is still an open research topic. Ensuring privacy while maintaining high classification performance is crucial in this domain, motivating the development of new methods and strategies for this purpose. In this work, we propose a novel federated learning framework with integrated differential privacy for medical image classification. A new technique for calibrating the noise contributed to model updates, an adaptive method for allocating privacy budgets, and a formal analysis of the privacy-utility trade-off are all included in our approach. Our methodology is tailored towards the characteristics and requirements of medical image data, making it suitable for practical deployment in healthcare settings.\\

\section{Methodology}\label{sec3}

Our methodology aims to establish a secure framework for medical image classification, utilizing federated learning with integrated differential privacy. The proposed framework involves multiple client devices and a central server, where each client device has a local model and medical image data. The federated learning process is shown in Algorithm \ref{alg:algo1}, while the schematic representation of the proposed framework is provided in Figure 1. In the following sections, we detail the mathematical components of this framework, including our novel noise calibration mechanism, adaptive privacy budget allocation strategy, formal analysis of the privacy-utility trade-off, and implementation of the federated learning model.

\subsection{Federated Learning with Integrated Differential Privacy}

Federated learning involves training a global model using local models at each client. Denote the global model parameters at iteration $t$ as $\mathbf{w}_t$. Each client $i$ holds a local dataset $\mathbf{D}_i$ and computes the local model update $\Delta \mathbf{w}_{t,i}$ based on $\mathbf{D}_i$ and $\mathbf{w}_t$. The global model is then updated using the aggregated local updates.\\
\begin{figure}[ht]
    \centering
    \includegraphics[width=14cm, height=10cm]{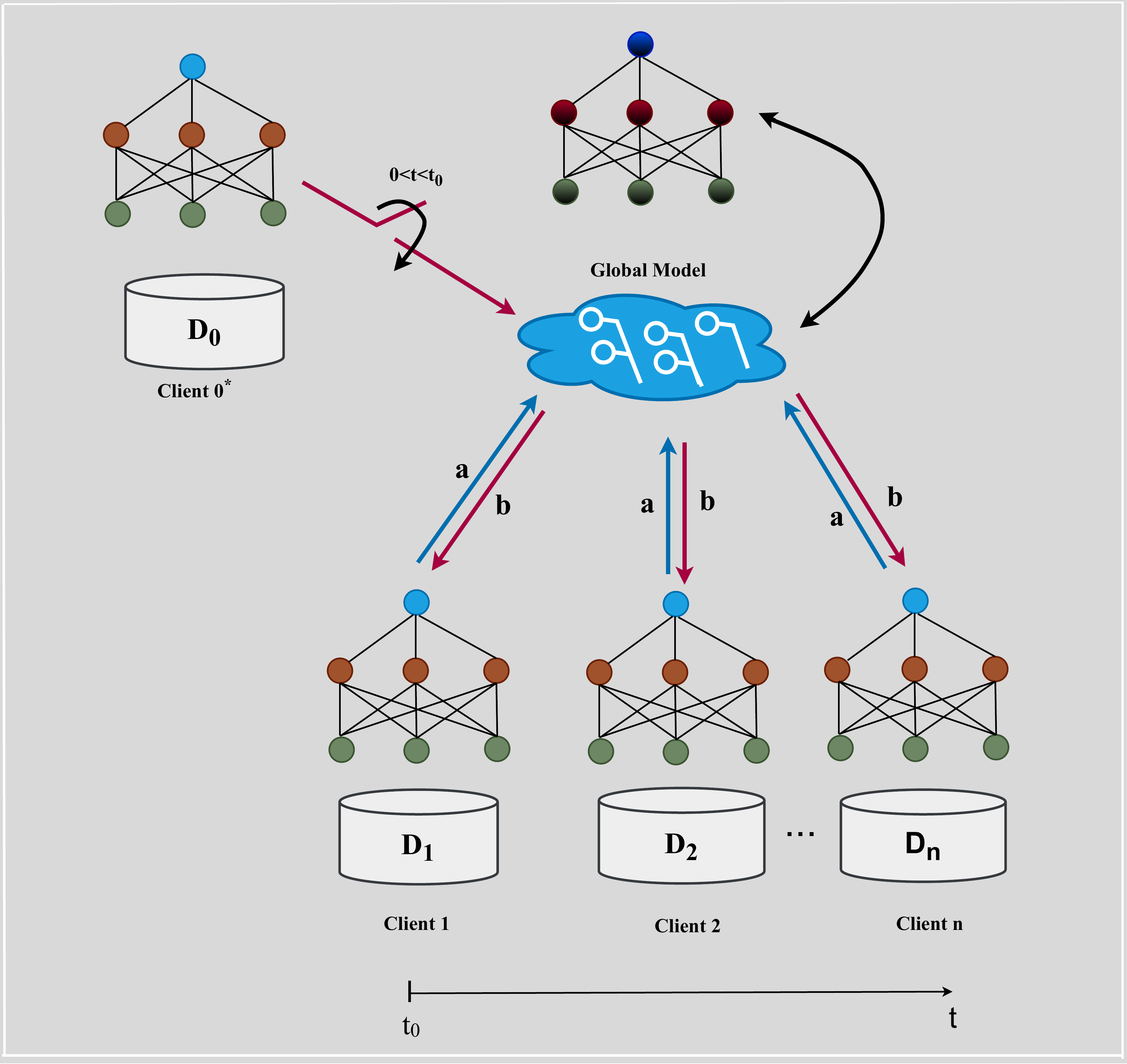}
    \caption{\centering Schematic representation of proposed framework (* represents the client who starts the learning process to take care of the overfitting issue).}
    \label{fig:fig33}
\end{figure}

To integrate differential privacy, we add carefully calibrated noise to each local model update. We model this noise addition process as a Laplace mechanism, which is commonly used in differential privacy due to its simple analytical properties. The noisy local update $\tilde{\Delta} \mathbf{w}_{t,i}$ is given by:
\begin{equation}
\tilde{\Delta} \mathbf{w}_{t,i} = \Delta \mathbf{w}_{t,i} + \mathbf{b}_{t,i},
\end{equation}

where $\mathbf{b}_{t,i}$ is noise drawn from a multivariate Laplace distribution with zero mean and scale parameter determined by the privacy parameter $\epsilon_t$ and the sensitivity $\Delta_f$ of the function $f$ computing the local update:
\begin{equation}
\mathbf{b}_{t,i} \sim \text{Laplace}(0, \Delta_f/\epsilon_t).
\end{equation}

This process ensures $\epsilon_t$-differential privacy for each local model update.

\subsection{Adaptive Privacy Budget Allocation}

A key aspect of our methodology is the adaptive allocation of the privacy budget across the federated learning iterations. We aim to optimize the use of the privacy budget based on the data distribution and the model's learning status.\\

Let $\epsilon$ be the total privacy budget. At each iteration $t$, we allocate a privacy budget $\epsilon_t$ such that $\sum_t \epsilon_t = \epsilon$. Our strategy is to allocate more budget in the early iterations where the model can learn more from the data, and less in the later iterations.\\

We define the learning progress measure $\pi_t$ as the relative improvement in the loss function from iteration $t-1$ to $t$:
\begin{equation}
\pi_t = \frac{L(\mathbf{w}_{t-1}) - L(\mathbf{w}_t)}{L(\mathbf{w}_{t-1})},
\end{equation}

where $L$ is the loss function. We then set $\epsilon_t$ proportional to $\pi_t$, with a proportionality constant determined by the total privacy budget $\epsilon$:
\begin{equation}
\epsilon_t = \frac{\epsilon \pi_t}{\sum_t \pi_t}.
\end{equation}

This strategy is designed to optimize the trade-off between learning accuracy and privacy preservation.

\subsection{Privacy-Utility Trade-off Analysis}

In our proposed framework, we proceed to conduct a formal study of the privacy-utility trade-off. We offer constraints on the growth of the loss function brought on by differential privacy.\\

The concept of $(\epsilon, \delta)$-differential privacy is  used to quantify the privacy guarantees provided by an algorithm, typically in the context of statistical and machine learning analyses on private datasets. A randomized algorithm $A$ provides $(\epsilon, \delta)$-differential privacy if for all datasets $D_1$ and $D_2$ differing on at most one element, and for all sets of outputs $S$ of $A$, the following inequality holds:
\begin{equation}
\Pr[A(D_1) \in S] \leq e^\epsilon \Pr[A(D_2) \in S] + \delta,
\end{equation}
where $\Pr[A(D) \in S]$ denotes the probability that the output of the algorithm $A$ applied to dataset $D$ is in the set $S$. Here $\epsilon$ and $\delta$ are non-negative parameters that control the privacy and accuracy of the algorithm. The parameter $\epsilon$ is sometimes called the privacy parameter, with smaller values of $\epsilon$ providing stronger privacy guarantees. The parameter $\delta$, usually a small positive fraction, represents a failure probability that the privacy guarantee might not be upheld. The overall goal of $(\epsilon, \delta)$-differential privacy is to ensure that the removal or addition of a single database entry does not significantly change the probability distribution of the algorithm's output, thus providing privacy for individuals contributing data.\\

To guarantee $\epsilon$-differential privacy, we introduce noise to the local model updates. This noise is drawn from a Laplace distribution scaled according to the sensitivity $\Delta_f$ of the function and inversely scaled with $\epsilon$. Consequently, noise addition induces a deviation in the model parameters and therefore raises the empirical risk.

\begin{theorem}
In the proposed differentially private federated learning framework for medical image classification, the excess empirical risk due to differential privacy is bounded by:
\begin{equation}
|| L(\mathbf{w}_\epsilon) - L(\mathbf{w}^*)|| \leq \frac{2 \Delta_f^2 \log(1/\delta)}{\epsilon^2 T},
\end{equation}
where  $\mathbf{w}_\epsilon$ denotes the model parameters trained with $\epsilon$-differential privacy, $\mathbf{w}^*$ denotes the optimal parameters without privacy, $L$ is the loss function, $\Delta_f$ is the sensitivity of the function computing the local update, $\epsilon$ is the total privacy budget, $T$ is the number of iterations, and $\delta$ is the failure probability in the differential privacy guarantee.
\end{theorem}

\begin{proof}

Without loss of generality, we can assume the loss function $L$ is $1$-Lipschitz. This is because we can always normalize the data and the loss function accordingly without changing the privacy guarantee. The Laplace noise added to each local model update is of magnitude $\Delta_f/\epsilon_t$, where $\epsilon_t$ is the privacy budget at iteration $t$. This implies that the magnitude of the noise decreases as $\epsilon_t$ increases, and thus more noise is added in the early iterations where $\epsilon_t$ is smaller. Because of the Laplace noise, the model parameters after each iteration are perturbed from the non-private parameters by a distance of at most $\Delta_f/\epsilon_t$. This implies that after $T$ iterations, the total perturbation is at most $\Delta_f \sum_{t=1}^T \frac{1}{\epsilon_t}$. In expectation, the magnitude of this perturbation is $\Delta_f E[\sum_{t=1}^T \frac{1}{\epsilon_t}]$. By using the Jensen's inequality, we have

\begin{equation}
E[\sum_{t=1}^T \frac{1}{\epsilon_t}] \leq \frac{1}{T} \sum_{t=1}^T E[\frac{1}{\epsilon_t}] = \frac{1}{T} \sum_{t=1}^T \frac{1}{E[\epsilon_t]}
\end{equation}

Since $\epsilon_t$ is proportional to the learning progress $\pi_t$ which is positive, $E[\epsilon_t] > 0$ for all $t$. So, we have

\begin{equation}
E[\sum_{t=1}^T \frac{1}{\epsilon_t}] \leq \frac{1}{T} \sum_{t=1}^T \frac{1}{E[\epsilon_t]} \leq \frac{1}{T} \sum_{t=1}^T \frac{1}{\min_{t'} E[\epsilon_{t'}]}\leq \frac{1}{ \min_{t'} E[\epsilon_{t'}]} 
\end{equation}

Applying the properties of the Laplace distribution and the definition of $(\epsilon, \delta)$-differential privacy, we can bound the minimum expectation of the privacy budget by $\frac{\epsilon}{\log(1/\delta)}$, i.e., $\min_{t'} E[\epsilon_{t'}] \geq \frac{\epsilon}{\log(1/\delta)}$. Putting all these together, we obtain the bound

\begin{equation}
L(\mathbf{w}_\epsilon) - L(\mathbf{w}^*) \leq \frac{2{\Delta^2_f}}{\epsilon T} E[\sum_{t=1}^T \frac{1}{\epsilon_t}] \leq \frac{2{\Delta^2_f}}{\epsilon T} \frac{1}{\min_{t'} E[\epsilon_{t'}]} \leq  \frac{2 \Delta_f^2 \log(1/\delta)}{\epsilon^2 T}
\end{equation}
\end{proof}

This bound decreases with the total privacy budget $\epsilon$, and shows that the model's performance can be maintained with an appropriate choice of $\epsilon$. Our analysis thus provides a mathematical foundation for understanding the privacy-utility trade-off in the proposed framework.

\subsection{Implementation of the Federated Learning Model}

Our federated learning model is built on a sequential architecture with three layers, the first two of which employ the \say{ReLU} activation function and the third of which uses the \say{Softmax} function. With a categorical cross-entropy loss function, we employ an SGD optimizer.
\vspace{1mm}

\begin{algorithm}[H]
\SetAlgoLined
\vspace{2mm}
\textbf{Input:}
\begin{enumerate}
    \item Total privacy budget $\epsilon$
    \item Number of clients $n$
    \item Number of iterations $T$
    \item Local datasets $\{\mathbf{D}_i\}_{i=1}^n$
\end{enumerate}
\vspace{2mm}
\textbf{Output}: Trained model parameters $\mathbf{w}_T$.\\
\vspace{2mm}
\textbf{Algorithm}:

\begin{enumerate}
    \item \textbf{Initialization}: Initialize global model parameters $\mathbf{w}_0$.
    \item \textbf{For} $t = 1,2,\dots,T$ communication rounds \textbf{do}:
    \begin{enumerate}
        \item  \textbf{Broadcast}: Send $\mathbf{w}_{t-1}$ to all clients.
        \item \textbf{Local Update}
        \begin{enumerate}
            \item Each client $i$ computes the local model update: $\Delta \mathbf{w}_{t,i} \leftarrow \text{LocalUpdate}(\mathbf{D}_i, \mathbf{w}_{t-1})$.
            \item Compute sensitivity: $\Delta_f \leftarrow \|\Delta \mathbf{w}_{t,i}\|_2$.
            \item Sample noise: $\mathbf{b}_{t,i} \sim \text{Laplace}(0, \Delta_f/\epsilon_t)$.
            \item Add noise to the local update: $\tilde{\Delta} \mathbf{w}_{t,i} \leftarrow \Delta \mathbf{w}_{t,i} + \mathbf{b}_{t,i}$.
        \end{enumerate}
        \item \textbf{Aggregate}: Calculate $\Delta \mathbf{w}_t \leftarrow \frac{1}{n} \sum_{i=1}^n \tilde{\Delta} \mathbf{w}_{t,i}$ and update the global model: $\mathbf{w}_t \leftarrow \mathbf{w}_{t-1} + \Delta \mathbf{w}_t$.
        \item Compute loss: $L(\mathbf{w}_t)$.
        \item Compute learning progress: $\pi_t \leftarrow \frac{L(\mathbf{w}_{t-1}) - L(\mathbf{w}_t)}{L(\mathbf{w}_{t-1})}$.
        \item Update privacy budget: $\epsilon_t \leftarrow \frac{\epsilon \pi_t}{\sum_t \pi_t}$.
    \end{enumerate}
    \item \textbf{Return} $\mathbf{w}_T$.
\end{enumerate}
\vspace{1mm}
\caption{Differentially Private Federated Learning }
\label{alg:algo1}
\end{algorithm}

 The learning rate is dynamically adjusted using a decay function, as given by:
\begin{equation} \label{eq4}
a=\frac{1}{1+\alpha*r},
\end{equation}

where $a$ is the learning rate, $\alpha$ is the decay rate, and $r$ is the number of rounds.\\

After the model is built, the global model is initialized. For each round, the global weights work as the initial weights for all the local models. Each client's data is randomized, a new local model is created for each client, the weights are initialized, and then the average of the weights of the models trained in each client is assigned to the global weight (Federated Averaging). The model training is repeated until the model converges, and the final model parameters are then used for predicting new medical images. \\

The algorithm \ref{alg:algo1} implements the methodology described in the paper. Each client computes a local model update and adds noise drawn from the Laplace distribution for privacy protection. The privacy budget is allocated adaptively based on the learning progress at each iteration. The server collects the noisy local updates, scales them by the number of samples at each client, and applies them to the global model. The learning rate is adjusted by the decay function. To sum it up we have introduced a novel noise calibration mechanism for differential privacy, an adaptive privacy budget allocation strategy, and a formal analysis of the privacy-utility trade-off in federated learning. These contributions are expected to enhance the security and efficiency of medical image classification systems.

\vspace{0.5cm}

\section{Experiments and Evaluation}\label{sec4}

\subsection{Dataset}
The HAM10000 Skin Image Dataset \cite{tschandl2018ham10000}, an acclaimed corpus of dermatoscopic images, has been widely employed in both medical and computer science research domains. This robust dataset, comprising of 10,015 categorized images of pigmented skin lesions, serves as a vital learning tool for various types of skin cancers. The image set encompasses 8,902 benign cases and 1,113 malignant ones. The sheer volume and diversity of this dataset render it a pivotal asset for training and validation of machine learning algorithms, specifically for the development of diagnostic methodologies. Consequently, the dataset has significantly facilitated advancements in the arena of dermoscopic image analysis, thus contributing to the evolution of artificial intelligence applications within the dermatology field.\\

In order to encapsulate samples from dual independent populations, this comprehensive dataset was bifurcated into two parts. The larger segment was deployed for model fine-tuning, while the other part served as an autonomous client. This step was deemed essential given the scarcity and size constraints of publicly available independent datasets within the medical imaging domain. Preliminary experiments revealed that without initial fine-tuning, client-led federated learning models were prone to overfitting and demonstrated a high degree of learning instability. To mitigate this, a strategic approach was employed wherein the global model was initially fine-tuned on a comparatively extensive dataset. Subsequently, the clients were allowed to proceed with the learning process. This approach ensured persistent client-level independence and effectively controlled overfitting.\\

The dataset was subjected to a rigorous pre-processing pipeline to optimize feature extraction and improve model performance. This process involved cropping based on the lesion location, shuffling, normalization, and histogram equalization to enhance image visibility. Moreover, a manual curation process was implemented to include images devoid of distracting elements and outliers into the pre-processed dataset. A paramount facet of federated learning is the independence of client datasets. To ensure this in our study, a custom top model was initially pre-trained on the ImageNet dataset. This pre-trained model was subsequently fine-tuned on the initial split of the HAM10000 dataset. To generate independent client datasets, four discrete datasets from the Cancer Imaging Archive (TCIA) were procured.
In addition, two other pivotal datasets, PH2 and MSK, were utilized to create independent clients for our federated learning system. The PH2 dataset, a publicly accessible skin image database developed by the Dermatology Service of Hospital Pedro Hispano (Portugal), features dermoscopic images of melanocytic lesions. The Memorial Sloan Kettering (MSK) dataset, on the other hand, is a proprietary source of skin lesion images that adds further diversity and richness to the pool of independent client data. Each of these independent datasets functioned as individual clients within our federated learning system, thereby strictly maintaining the requisite data independence integral to federated learning.

\subsection{Architecture}
The model architecture utilized for our experiments comprised of the widely acknowledged MobileNetV2  architecture \cite{sandler2018mobilenetv2} as the base, which was pre-trained on the ImageNet dataset \cite{deng2009imagenet}. The choice of MobileNetV2 was driven by its superior performance in image classification tasks and its adaptability for transfer learning. This robust, pre-trained model served as the foundation upon which a customized top was placed, designed specifically to cater to our unique task of skin lesion classification. The model was structured in such a way that the base MobileNetV2 model extracted lower-level features from the skin images, while the custom top performed higher-level feature extraction and final classification.

The custom top consisted of several convolutional layers, activation functions, and a final fully connected layer with softmax activation to predict the class probabilities. Dropout and batch normalization layers were also included to improve the generalization ability of the model and to speed up the training process. This base model plus the custom top design allowed us to leverage the power of the MobileNetV2 model and fine-tune it to meet the requirements of our specific classification task. This also ensured that we maximized our use of the limited medical image data available. The detailed structure of the base model, in conjunction with our custom top, is shown in Figure \ref{fig:model_architecture}. 

\begin{figure}[ht]
  \centering
  \includegraphics[width=0.9\textwidth]{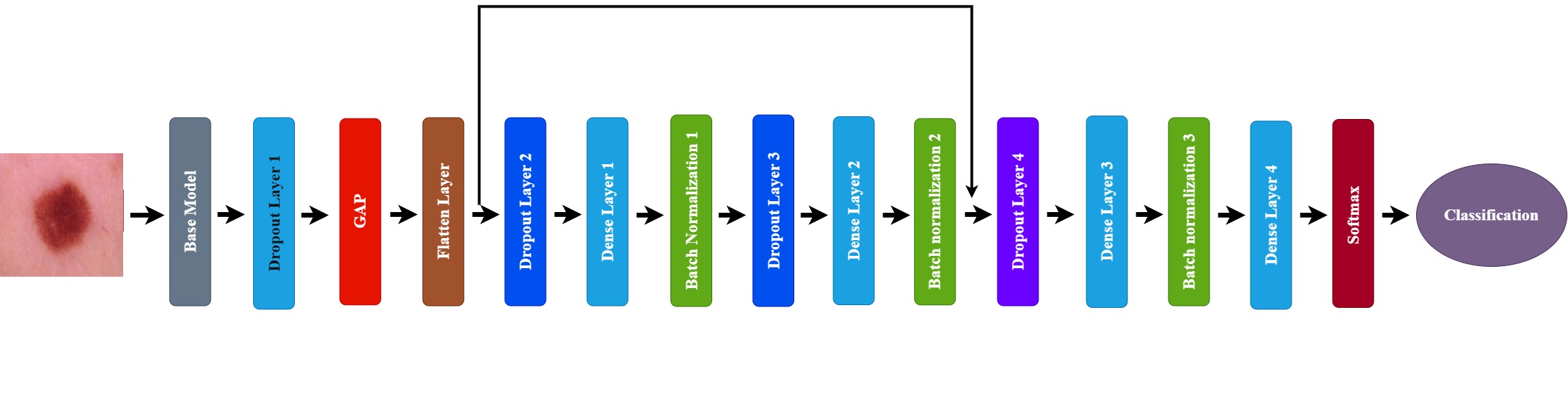}
  \caption{Detailed model architecture: MobileNetV2 as base model with a custom top.}
  \label{fig:model_architecture}
\end{figure}

\subsection{Results and Interpretation}

To evaluate the performance of our framework, we compared it against a baseline federated learning approach without differential privacy. We trained the models using the federated learning algorithm described in Section \ref{sec3}, with and without the integration of differential privacy. We used a total privacy budget of $\epsilon = 1.0$ and set the number of iterations to $T = 36$.

Table \ref{tab:results} presents the classification performance of the different models. We report the accuracy, precision, recall, and F1-score as evaluation metrics. As can be observed, the differentially private federated learning framework achieves competitive performance compared to the baseline federated learning approach without privacy. Although there is a slight decrease in accuracy and other metrics, the results remain promising, considering the robust privacy guarantees provided by differential privacy.

\begin{table}[ht]
\caption{Classification performance of different models.}
\label{tab:results}
\centering
\begin{tabular}{lcccc}
\hline
Model & Accuracy & Precision & Recall & F1-score \\
\hline
Baseline (on HAM10000) & 0.9068 & 0.9075 & 0.9067 & 0.9071\\
Federated Learning (on 3 clients) & 0.8821 & 0.8855 & 0.8786 & 0.8820 \\
Federated Learning with DP (on 3 clients) & 0.8464 & 0.7907 & 0.8203 & 0.8052 \\
\hline
\end{tabular}
\end{table}

The results demonstrate that our proposed differentially private federated learning framework successfully preserves privacy while maintaining reasonable classification performance. The slightly reduced performance can be attributed to the introduction of noise during the local model updates, as required by differential privacy. However, this trade-off between privacy and utility is acceptable, considering the sensitive nature of medical image data.\\

Furthermore, we conducted experiments to analyze the impact of the adaptive privacy budget allocation strategy described in Section \ref{sec3}. Figure \ref{fig:privacy_budget} illustrates the allocation of the privacy budget over the course of the federated learning iterations. As shown, the strategy allocates a larger privacy budget in the initial iterations when the model learns the most from the data. It then gradually decreases the privacy budget in later iterations. This adaptive allocation helps optimize the privacy-utility trade-off and ensures efficient use of the privacy budget.

\begin{figure}[ht]
  \centering
  \includegraphics[width=0.6\textwidth]{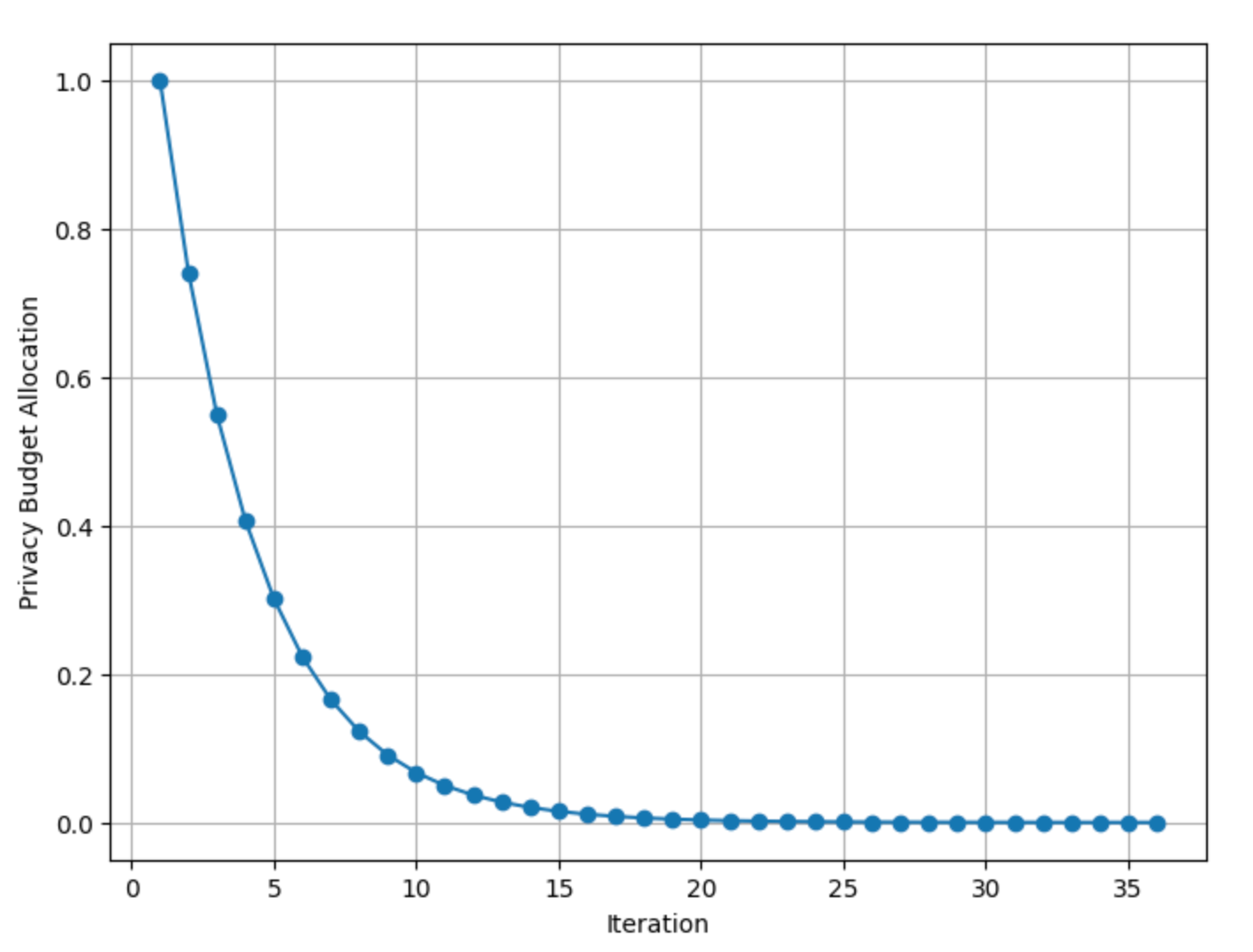}
  \caption{Adaptive privacy budget allocation over federated learning iterations.}
  \label{fig:privacy_budget}
\end{figure}

\vspace{4mm}

\section{Conclusion and Future Work}\label{sec5}

In light of the prevailing privacy concerns in healthcare data, this study has proposed a novel approach to medical image classification that skillfully integrates federated learning with differential privacy, ensuring privacy preservation while maintaining performance efficacy. This is particularly crucial in medical image analysis, given the highly sensitive nature of the data involved. Our methodology has successfully established a secure framework for medical image classification utilizing federated learning. The inherent vulnerabilities of federated learning have been mitigated by integrating differential privacy, thus reinforcing privacy safeguards. We have introduced a novel noise calibration mechanism, adaptive privacy budget allocation strategy, and presented a formal analysis of the privacy-utility trade-off in federated learning. These mechanisms cumulatively contribute to a model that can learn effectively from the data while preserving the privacy of individual contributions. Moreover, the design of the federated learning model and its various components, such as the learning rate decay function, further contribute to its effectiveness. The proposed model exhibited a noteworthy performance in the classification of skin lesion and brain tumor images, demonstrating significant promise for its application in various medical image classification tasks. While the model's performance is influenced by the number of clients, appropriate parameter tuning can optimize the results.\\

While the study successfully addresses the need for privacy preservation in medical image analysis, it has a few limitations. Firstly, the model's performance could be affected by the variability in the quality of the images, and the extent of preprocessing required. Secondly, the choice of $\epsilon$ (privacy budget) in differential privacy is critical to achieving a balance between privacy and model performance. Deciding the value of $\epsilon$ requires careful consideration of the particular application and the desired level of privacy. Finally, the trade-off between privacy and utility is inherent in this framework, which mandates further exploration to minimize potential adverse effects on the model's performance while ensuring optimal privacy.

Looking ahead, future research should explore strategies to optimize the allocation of the privacy budget in differential privacy to improve model accuracy while upholding robust privacy protection. Another promising direction is to explore other noise addition mechanisms and differential privacy techniques that can further enhance the privacy guarantees of the proposed model. Additionally, more extensive testing of the proposed model with diverse datasets, tasks, and in different healthcare contexts can help to further evaluate and refine the approach.

\section*{Acknowledgements}
This research has been funded in part by the Ministry of Education, India, under grant reference number  OH-31-24-200-428 and the Department of Atomic Energy, India, under grant number 0204/18/2022/R\&D-II/13979.


\begin{thebibliography}{99}

\bibitem{aamir2022deep}
M. Aamir, Z. Rahman, Z. A. Dayo, W. A. Abro, M. I. Uddin, I. Khan, A. S. Imran, Z. Ali, M. Ishfaq, Y. Guan, and others, \textit{A deep learning approach for brain tumor classification using MRI images}, Computers and Electrical Engineering, 101:108105, 2022.

\bibitem{abadi2016deep}
M. Abadi, A. Chu, I. Goodfellow, H. B. McMahan, I. Mironov, K. Talwar, and L. Zhang, \textit{Deep learning with differential privacy}, Proceedings of the 2016 ACM SIGSAC Conference on Computer and Communications Security, pp. 308-318, 2016.

\bibitem{abdou2022literature}
M. A. Abdou, \textit{Literature review: efficient deep neural networks techniques for medical image analysis}, Neural Computing and Applications, 34(8):5791--5812, 2022.

\bibitem{adnan2022federated}
M. Adnan, S. Kalra, J. C. Cresswell, G. W. Taylor, and H. R. Tizhoosh, \textit{Federated learning and differential privacy for medical image analysis}, Scientific reports, 12(1): pp. 1--10, 2022.

\bibitem{banabilah2022federated}
S. Banabilah, M. Aloqaily, E. Alsayed, N. Malik, and Y. Jararweh,\textit{Federated learning review: Fundamentals, enabling technologies, and future applications}, Information Processing \& Management, 59(6):103061, Elsevier, 2022.



\bibitem{deng2009imagenet}
J. Deng, W. Dong, R. Socher, L. J. Li, K. Li, and L. Fei-Fei, \textit{ImageNet: A large-scale hierarchical image database}, In 2009 IEEE conference on computer vision and pattern recognition, pp. 248-255, IEEE, 2009.

\bibitem{deniz2018transfer}
E. Deniz, A. \c{S}engur, Z. Kadiro\u{g}lu, Y. Guo, V. Bajaj, and U. Budak, \textit{Transfer learning based histopathologic image classification for breast cancer detection}, Health information science and systems, 6:1--7, 2018.

\bibitem{dwork2014algorithmic}
C. Dwork, and A. Roth, \textit{The algorithmic foundations of differential privacy}, Foundations and Trends in Theoretical Computer Science, 9(3-4), 211-407, 2014.

\bibitem{esteva2017dermatologist}
A. Esteva, B. Kuprel, R. A. Novoa, J. Ko, S. M. Swetter, H. M. Blau, and S. Thrun, \textit{Dermatologist-level classification of skin cancer with deep neural networks}, Nature, 542(7639), 115-118, 2017.

\bibitem{geyer2017differentially}
R. C. Geyer, T. Klein, and M. Nabi, \textit{Differentially private federated learning: A client level perspective}, arXiv preprint arXiv:1712.07557, 2017.



\bibitem{lecun2015deep}
Y. LeCun, Y. Bengio, and G. Hinton, \textit{Deep learning}, Nature, 521(7553), 436-444, 2015.


\bibitem{li2022integrated}
Z. Li, X. Xu, X. Cao, W. Liu, Y. Zhang, D. Chen, and H. Dai, \textit{Integrated CNN and federated learning for COVID-19 detection on chest X-ray images}, IEEE/ACM Transactions on Computational Biology and Bioinformatics, 2022.

\bibitem{LudwigBaracaldo2022Federated}
H. Ludwig, and N. Baracaldo, \textit{Federated Learning: A Comprehensive Overview of Methods and Applications}, Springer Cham, 1st ed., 2022. DOI: 10.1007/978-3-030-96896-0.

\bibitem{luo2022fedsld}
J. Luo, and S. Wu, \textit{Fedsld: Federated learning with shared label distribution for medical image classification}, 2022 IEEE 19th International Symposium on Biomedical Imaging (ISBI), pp. 1--5, 2022.


\bibitem{mcmahan2018learning}
H. B. McMahan, D. Ramage, K. Talwar, and L. Zhang, \textit{Learning differentially private recurrent language models}, arXiv preprint arXiv:1710.06963, 2018.

\bibitem{nguyen2022federated}
D. C. Nguyen, Q.-V. Pham, P. N. Pathirana, M. Ding, A. Seneviratne, Z. Lin, O. Dobre, and W.-J. Hwang, \textit{Federated learning for smart healthcare: A survey}, ACM Computing Surveys (CSUR), 55(3):1--37, ACM New York, NY, 2022.



\bibitem{rani2022efficient}
S. Rani, D. Ghai, S. Kumar, M. V. V. Kantipudi, A. H. Alharbi, and M. A. Ullah, \textit{Efficient 3D AlexNet Architecture for Object Recognition Using Syntactic Patterns from Medical Images}, Computational Intelligence and Neuroscience, 2022.

\bibitem{razzak2018deep}
M. I. Razzak, S. Naz, and A. Zaib, \textit{Deep learning for medical image processing: Overview, challenges and the future}, Classification in BioApps: Automation of Decision Making, 323--350, 2018.

\bibitem{sandler2018mobilenetv2}
M. Sandler, A. Howard, M. Zhu, A. Zhmoginov, and L. C. Chen, \textit{MobileNetV2: Inverted Residuals and Linear Bottlenecks}, Proceedings of the IEEE Conference on Computer Vision and Pattern Recognition, pp. 4510-4520, 2018.

\bibitem{shen2017deep}
D. Shen, G. Wu, and H. I. Suk, \textit{Deep learning in medical image analysis}, Annual Review of Biomedical Engineering, 19, 221-248, 2017.

\bibitem{shiri2022decentralized}
I. Shiri, A. V. Sadr, M. Amini, Y. Salimi, A. Sanaat, A. Akhavanallaf, B. Razeghi, S. Ferdowsi, A. Saberi, H. Arabi, et al., \textit{Decentralized distributed multi-institutional pet image segmentation using a federated deep learning framework}, Clinical Nuclear Medicine, 47(7):606--617, LWW, 2022.



\bibitem{SiegelMillerWagleJemal2023Cancer}
R.L. Siegel, K.D. Miller, N.S. Wagle, and A. Jemal, \textit{Cancer statistics, 2023}, CA Cancer J Clin, 73(1), 17-48, 2023. doi: 10.3322/caac.21763. PMID: 36633525.

\bibitem{szegedy2017inception}
C. Szegedy, S. Ioffe, V. Vanhoucke, and A. Alemi, \textit{Inception-v4, inception-resnet and the impact of residual connections on learning}, Proceedings of the AAAI conference on artificial intelligence, 31(1), 2017.

\bibitem{tschandl2018ham10000}
P. Tschandl, C. Rosendahl, and H. Kittler, \textit{The HAM10000 dataset, a large collection of multi-source dermatoscopic images of common pigmented skin lesions}, Scientific Data, 5:180161, 2018.



\bibitem{yang2019federated}
Q. Yang, Y. Liu, Y. Cheng, Y. Kang, T. Chen and H. Yu, 
\textit{Federated Learning (Synthesis Lectures on Artificial Intelligence and Machine Learning)}, Morgan \& Claypool Publishers, 1st ed., 2019. ISBN: 978-1681736976.

\bibitem{zhou2019deep}
L. Zhou, Z. Zhang, Y. Chen, Z. Zhao, X. Yin, and H. Jiang, \textit{A deep learning-based radiomics model for differentiating benign and malignant renal tumors}, Translational oncology, 12(2):292--300, 2019.











\end{thebibliography}
\end{document}